\documentclass[11pt]{article}

\usepackage[utf8]{inputenc}
\usepackage{amsmath, amssymb, amsthm}
\usepackage{fullpage}
\usepackage{hyperref}
\usepackage[round, comma]{natbib}
\usepackage[usenames]{color}
\usepackage{times}

\newcommand{\vknote}[1]{{\color{blue} (VK) #1}}
\newcommand{\eat}[1]{ }

\def\ie{{i.e.},~}
\def\eg{{e.g.},~}
\def\etal{{et al.}~}

\def\moo{\{-1, 1\}}
\def\reals{\mathbb{R}}

\def\true{\textsc{True}}
\def\false{\textsc{False}}
\def\E{\mathbb{E}}
\def\Pol{\mathbb{P}}
\def\chop{\mathrm{chop}}
\def\sign{\mathrm{sgn}}
\def\weight{\mathrm{weight}}
\def\eye{\mathbb{I}}

\def\DISJ{\mathrm{DISJ}}
\def\VC{\mathrm{VC\mbox{-}DIM}}

\def\EX{\mathsf{EX}}
\def\opt{\operatorname{opt}}
\def\err{\operatorname{err}}
\def\falsep{\operatorname{false}_+}
\def\falsen{\operatorname{false}_-}
\def\adegp{\widetilde{\deg}_+}
\def\adeg{\widetilde{\deg}}
\def\adegpe{\widetilde{\deg}_{+, \eps}}
\def\adegn{\widetilde{\deg}_-}
\def\adegne{\widetilde{\deg}_{-, \eps}}
\def\eps{\epsilon}
\def\classF{{\mathcal F}}
\def\rad{{\mathcal R}}
\def\sample{{\mathcal S}}
\def\Eloss{{\mathcal L}}

\newcommand{\linfnorm}[1]{\Vert #1 \Vert_{\infty}}
\newcommand{\lonenorm}[1]{\Vert #1 \Vert_{1}}

\DeclareMathOperator{\sgn}{sgn}
\DeclareMathOperator{\MAJ}{MAJ}
\DeclareMathOperator{\OR}{OR}
\DeclareMathOperator{\AND}{AND}
\DeclareMathOperator{\poly}{poly}

\newtheorem{theorem}{Theorem}

\newtheorem{definition}{Definition}
\newtheorem{remark}{Remark}
\newtheorem{corollary}{Corollary}

\newtheorem{fact}{Fact}

\title{Distribution-Independent Reliable Learning}
\author{Varun Kanade\thanks{University of California, Berkeley. Email:
\texttt{vkanade@eecs.berkeley.edu}} \and 
Justin Thaler\thanks{The Simons Institute for the Theory of Computing at UC
Berkeley. Email: \texttt{jthaler@seas.harvard.edu}}} 

\begin{document}

\date{}

\maketitle 

\begin{abstract}
We study several questions in the \emph{reliable agnostic} learning framework of
Kalai et al. (2009), which captures learning tasks in which one type of error is costlier than other types. A positive reliable classifier is one that makes no false
positive errors.  The goal in the \emph{positive reliable} agnostic framework is
to output a hypothesis with the following properties: (i) its false positive
error rate is at most $\epsilon$, (ii) its false negative error rate is at most
$\epsilon$ more than that of the best positive reliable classifier from the
class.  A closely related notion is \emph{fully reliable} agnostic learning,
which considers \emph{partial classifiers} that are allowed to predict ``unknown'' on
some inputs.  The best fully reliable partial classifier is one that
makes no errors and minimizes the probability of predicting ``unknown'', and the
goal in fully reliable learning is to output a hypothesis that is almost as good
as the best fully reliable partial classifier from a class.

For distribution-independent learning,  the best known algorithms for PAC
learning typically utilize polynomial threshold representations, while the state
of the art agnostic learning algorithms use point-wise polynomial
approximations.  We show that \emph{one-sided polynomial approximations}, an
intermediate notion between polynomial threshold representations and point-wise
polynomial approximations, suffice for learning in the reliable agnostic
settings. We then show that majorities can be fully reliably learned and
disjunctions of majorities can be positive reliably learned, through
constructions of appropriate one-sided polynomial approximations.  Our fully
reliable algorithm for majorities provides the first evidence that fully
reliable learning may be strictly easier than agnostic learning.  Our algorithms
also satisfy strong attribute-efficiency properties, and in many cases they
provide smooth tradeoffs between sample complexity and running time.



\end{abstract}
\newpage


\section{Introduction}
\label{sec:intro}
In many learning tasks, one type of error is costlier than other types. For
example, when detecting spam messages, an important mail marked as spam (a false
positive) is a major problem, whereas false negatives are only a minor nuisance.
On the other hand, in settings such as detecting failures in an electric
network, false negatives may be very harmful. In yet other settings, it may be
better to refrain from making a prediction at all, rather than make a wrong one,
\eg when detecting medical ailments. Following \citet{KKM:2012}, we call these
kinds of tasks \emph{reliable learning}. Closely related tasks have been widely
studied in the statistics and machine learning literature. We discuss some of
this work later; here, we simply note that the work of Kalai \etal and the
present work depart from much of the extant literature by emphasizing
computational considerations, \ie by focusing on ``fast'' algorithms, and
guarantees with respect to the \emph{zero-one} loss.

\citet{KKM:2012} introduced a formal framework to study reliable learning in the
\emph{agnostic} setting, which is a challenging model that captures the problem
of learning in the presence of adversarial classification noise.  In particular,
the goal of an agnostic learning algorithm is to produce a hypothesis that has
error that is at most $\epsilon$ higher than the best from a certain class. A
false positive error occurs when the true label is negative, but the hypothesis
predicts positive. Analogously, a false negative error occurs when the true
label is positive, but the hypothesis predicts negative.

The best positive reliable classifier from a class is one that make no false
positive errors and minimizes false negative errors. In the \emph{positive
reliable} learning setting, the goal of a learning algorithm is to output a
hypothesis with the following properties: (i) its false positive error rate is
at most $\epsilon$, (ii) its false negative error rate is at most $\epsilon$
more than that of the best positive reliable classifier from the class. The
notion of \emph{negative reliable} learning is identical with the roles of false
positive and false negatives reversed.

\citet{KKM:2012} also introduced the notion of full reliability. A partial
classifier is one that is allowed to sometimes predict ``$?$'' or
\emph{unknown}. The best partial classifier from a class is one that makes no
errors and minimizes the probability of predicting $?$.  In the \emph{fully
reliable} learning setting, the goal of the learning algorithm is to output a
hypothesis $h : X \rightarrow \{-1, ?, +1 \}$ such that (i) the error of $h$ is
at most $\epsilon$, (ii) the probability that $h$ predicts `$?$' is at most
$\epsilon$ more than the best partial classifier from the class.
%

\subsection{Our Contributions}

%
In this work, we focus on distribution-independent reliable learning, and our
main technical contribution is to give new reliable learning algorithms for a
variety of concept classes.  We now place our reliable learning algorithms in
the context of prior work on PAC and agnostic learning. 

The \emph{threshold degree} of a Boolean function $f: \{-1, 1\}^n \rightarrow
\{-1, 1\}$ is the least degree of a real polynomial that agrees in sign with $f$
at all inputs $x \in \{-1, 1\}^n$. The \emph{approximate degree} (with error
parameter $\eps$) of $f$ is the least degree of a real polynomial that
point-wise approximates $f$ to error at most $\eps$.  It is well-known that
concept classes with low threshold degree can be efficiently learned in
Valiant's PAC model under arbitrary distributions; indeed, threshold degree
upper bounds underlie the fastest known PAC learning algorithms for a variety of
fundamental concept classes, including DNF and read-once formulae \citep{klivansservediodnfs, readonceformulae}. Meanwhile, concept
classes with low approximate degree can be efficiently learned in the agnostic
model, a connection that has yielded the fastest known algorithms for
distribution-independent agnostic learning \citep{KKMS:2005}.

We show that concept classes with low \emph{one-sided} approximate degree can be
efficiently learned in the \emph{reliable} agnostic model.  Here, one-sided
approximate degree is an intermediate notion that lies between threshold degree
and approximate degree; we defer a formal definition to Section
\ref{sec:onesidedpoly}.  One-sided approximate degree was introduced in its
present form by \cite{BT:2013} (see also \citep{sherstovthresholddegree}),
though equivalent \emph{dual} formulations had been used in several prior works
\citep{gavinskysherstov, andortree1, andortree2}. Our learning algorithm is
similar to the $L_1$ regression algorithm of \citet{KKMS:2005}; however, the
analysis of our algorithm is more delicate. Specifically, due to asymmetry in
the type of errors considered in the reliable setting, our analysis requires the use of \emph{two} loss functions.
On one side, we use the \emph{hinge} loss rather than $L_1$
loss, since one-sided polynomial approximations may be unbounded, and on the
other, we use a non-convex Lipschitz approximation to the \emph{zero-one} loss.

We identify important concept classes, such as majorities and intersections of
majorities, whose one-sided approximate degree is strictly smaller than its
approximate degree. Consequently, we obtain reliable (in the case of majorities,
even \emph{fully reliable}) agnostic learning algorithms that are strictly more
efficient than the fastest known algorithms in the standard agnostic setting.  Our
fully reliable learning algorithm for majorities gives the first indication that
fully reliable learning may be strictly easier than agnostic learning. Finally,
we show how to obtain smooth \emph{tradeoffs} between sample complexity and
runtime of algorithms for agnostically learning conjunctions, and for positive
reliably learning DNF formulae. 

In more detail, we summarize our new algorithmic results as follows (for
simplicity, we omit dependence on the error parameter $\eps$ of the learning
algorithm from this overview). We give:
\begin{itemize}
\item A simple $\poly(n)$ time algorithm for positive reliable learning of
disjunctions.
\item A $2^{\tilde{O}(\sqrt{n})}$ time algorithm for fully reliable learning of
majorities.  In contrast, no $2^{o(n)}$-time algorithm for agnostically learning majorities
is known in the arbitrary-distribution setting.
\item A $2^{\tilde{O}(\sqrt{n \log m})}$ time algorithm for positive (respectively,
negative) reliable learning of disjunctions (respectively, conjunctions) of $m$
majorities. 
\item For any $d > n^{1/2}$, a $n^{O(d)}$-time algorithm with sample complexity
$n^{O(n/d)}$ for agnostically learning conjunctions, and for positive reliably
learning $\poly(n)$-term DNFs.
\end{itemize}

All of our algorithms also satisfy very strong \emph{attribute-efficiency}
properties: if the function being learned depends on only $k \ll n$ of the $n$
input variables, then the sample complexity of the algorithm depends only
logarithmically on $n$, though the dependence on $k$ may be large. We defer a
detailed statement of these properties until Section \ref{sec:lp}.

\subsection{Related Work}
\label{sec:related}

The problem of reliable classification can be expressed as minimizing a loss
function with different costs for false negative and false positive errors (see
\eg \citep{Domingos:1999, Elkan:2001}).  Reliable learning is also related to
the Neyman-Pearson criterion from classical statistics --- where it has been
shown that the optimal strategy to minimize one type of errors, subject to the
other type being bounded, is to threshold the ratio of the
likelihoods~\citep{NP:1933}. However, the main problem is \emph{computational};
in general the loss functions with different costs from these prior works are
not convex and the resulting optimization problems are intractable. The work of
\citet{KKM:2012} and the present work departs from the prior work in that we
focus on algorithms with both provable guarantees on their generalization error
with respect to the \emph{zero-one} loss, and bounds on their computational
complexity, rather than focusing purely on statistical efficiency.

\citet{KKM:2012} showed that any concept class that is agnostically learnable
under a fixed distribution is also learnable in the reliable agnostic learning
models under the same distribution.  Furthermore, they showed that if a class
$C$ is agnostically learnable, the class of disjunctions of concepts in $C$ is
positive reliably learnable (and the class of conjunctions of concepts in $C$ is
negative reliably learnable). Finally, they showed that if $C$ is both positive
and negative reliably learnable, then it is also fully reliably learnable.
Using these general reductions, Kalai \etal showed that the class of
polynomial-size DNF formulae is positive reliable learnable under the uniform
distribution in polynomial time with membership queries (it also follows from
their reductions and the agnostic learning algorithm of \cite{KKMS:2005}
described below that DNF formulae can be positive reliably learned in the
distribution-independent setting in time $2^{\tilde{O}(\sqrt{n})}$).
Agnostically learning DNFs under the uniform distribution remains a notorious
open problem, and thus their work gave the first indication that positive (or
negative) reliable learning may be easier than agnostic learning.  

\cite{KKMS:2005} put forth an algorithm for agnostic learning based on
\emph{$L_1$-regression}.  Our reliable learning algorithms based on one-sided
approximate degree upper bounds is inspired by and generalizes their work.
\cite{klivanssherstovagnostic} subsequently established strong
\emph{limitations} on the $L_1$-regression approach of \cite{KKMS:2005}, proving
lower bounds on the size of \emph{any} set of ``feature functions'' that can
point-wise approximate the concept classes of majorities and conjunctions. Their
work implies that substantially new ideas will be required to obtain a
$2^{o(n)}$-time distribution-independent agnostic learning algorithm for
majorities, or a $2^{o(\sqrt{n})}$ time algorithm for agnostically learning
conjunctions.

Finally, lower bounds on one-sided approximate degree have recently been used in
several works to establish strong limitations on the power of existing
algorithms for PAC learning~\citep{BT:2013, sherstovthresholddegree,
andortree2, gavinskysherstov, andortree1}.  In this paper, we do the opposite:
we use one-sided approximate degree upper bounds to give new, more efficient
learning algorithms in the reliable agnostic setting. 


\subsubsection*{Organization}

In Section \ref{sec:setting}, we review the definitions of agnostic learning,
and positive, negative and fully reliable learning.  In
Section~\ref{sec:results}, we first give a very simple polynomial time algorithm
for positive reliable learning of disjunctions, before showing that appropriate
one-sided polynomial approximations for function classes result in efficient
reliable learning algorithms. In Section~\ref{sec:one-sided}, we give
constructions of one-sided approximating polynomials for (conjunctions and
disjunctions of) low-weight halfspaces, as well as for DNF and CNF formulae. In
Section~\ref{sec:tradeoff}, we show how tradeoffs may be obtained for some of
our results between sample complexity and running time, and in
Section~\ref{sec:limitations}, we describe some limitations of our approach. We
end with a discussion and directions for future work.

\section{Preliminaries and Definitions}
\label{sec:setting}
Let $X = \moo^n$ denote the instance space. Let $C$ denote a concept class of
functions from $X \rightarrow \moo$.  We will use the convention that $+1$ is
$\true$ and $-1$ is $\false$.\footnote{This is contrary to the usual convention
in the analysis of Boolean functions. However, our definitions would appear a
lot more counter-intuitive in the standard notation.} For ease of notation, we
will keep the parameter $n$, corresponding to the length of input vectors,
implicit in the discussion. Let $c, h : X \rightarrow \moo$ be Boolean
functions. For a distribution $\mu$ over $X$, let $\err(h,(\mu, c)) = \Pr_{x
\sim \mu} [c(x) \neq h(x)]$, denote the error of hypothesis $h$ with respect to
concept $c$ and distribution $\mu$. Let $\EX(c, \mu)$ denote the example oracle,
which when queried returns a pair $(x, c(x))$, where $x$ is drawn from
distribution $\mu$, and $c$ is a concept in $C$. Since the algorithms presented
in this paper typically do not run in polynomial time, we do not impose such a
condition in the definitions of learnability. We will explicitly mention the
running time and sample complexity in all of our results.

\begin{definition}[PAC Learning~\citep{Valiant:1984}] A concept class $C$ is
probably approximately correct (PAC) learnable if there exists a learning
algorithm that for any $c \in C$, any distribution $\mu$ over $X$, any $\epsilon,
\delta > 0$, with access to an example oracle $\EX(c, \mu)$, outputs a hypothesis
$h$, such that with probability at least $1 - \delta$, $\err(h, (\mu, c)) \leq
\epsilon$. \end{definition}

In the case of agnostic learning, the data may come from an arbitrary joint
distribution on examples and labels. Let $D$ denote a distribution over $X
\times \moo$. Also, let $\err(h, D) = \Pr_{(x, y) \sim D}[ h(x) \neq y]$,
denote the error of $h$ with respect to $D$, and let $\EX(D)$ denote the
example oracle which when queried returns $(x, y) \sim D$.

\begin{definition}[Agnostic Learning~\citep{Haussler:1992, KSS:1994}] A concept
class $C$ is agnostically learnable if there exists a learning algorithm that
for any distribution $D$ over $X \times \moo$, any $\epsilon, \delta > 0$, with
access to example oracle $\EX(D)$, outputs a hypothesis $h$, such that with
probability at least $1 - \delta$, $\err(h, D) \leq \opt + \epsilon$, where
$\opt = \displaystyle\min_{c \in C} \err(c, D)$.
\end{definition}

\subsection{Reliable Learning}

We review the various notions of reliable agnostic learning proposed by
\citet{KKM:2012}. As in the case of agnostic learning, the data comes from an
arbitrary joint distribution $D$ over $X \times \moo$. For a Boolean function,
$h : X \rightarrow \moo$, define the false positive error ($\falsep$) and the false
negative error ($\falsen$) with respect to $D$ as follows:
\begin{align*}
\falsep(h, D) &= \Pr_{(x, y) \sim D}[h(x) = 1 \wedge y = -1 ] \\
\falsen(h, D) &= \Pr_{(x, y) \sim D}[h(x) = -1 \wedge y = +1 ]
\end{align*}
Let $C$ denote the concept class of interest for learning. For a distribution
$D$, define the following:
\begin{align*}
C^+(D) &= \{ c \in C ~|~ \falsep(c, D) = 0 \} \\
C^-(D) &= \{ c \in C ~|~ \falsen(c, D) = 0 \}
\end{align*}
We call the concepts in $C^+$ (respectively, $C^-$) positive (respectively,
negative) reliable with respect to $D$. Below we define positive and negative
reliable learning. In short, positive reliable learning requires that the
learning algorithm produce a hypothesis that makes (almost) no false positive
errors, while simultaneously minimizing false negative errors. Likewise, in the
case of negative reliable learning, the learning algorithm must output a
hypothesis that makes (almost) no false negative errors, while simultaneously
minimizing false positive errors. Although the definitions of positive and
negative reliable learning are entirely symmetric, we define the two separately
for the sake of clarity.

\begin{definition}[Positive Reliable Learning~\citep{KKM:2012}] A concept class
$C$ is positive reliably learnable if there exists a learning algorithm that
for any distribution $D$ over $X \times \moo$, and any $\epsilon, \delta > 0$, when given
access to the example oracle $\EX(D)$, outputs a hypothesis $h$ that satisfies
the following with probability at least $1 - \delta$, 
\begin{enumerate}
\item $\falsep(h, D) \leq \epsilon$ 
\item $\falsen(h, D) \leq \opt^+ + \epsilon$, where $\opt^+ = \displaystyle\min_{c \in C^+(D)} \falsen(c, D)$ 
\end{enumerate}
We refer to $\eps$ as the \emph{error parameter} of the learning algorithm.
\end{definition}

\begin{definition}[Negative Reliable Learning~\citep{KKM:2012}] A
concept class $C$ is negative reliably learnable, if there exists a
learning algorithm that for any distribution $D$ over $X \times \moo$, any
$\epsilon, \delta > 0$, with access to the example oracle $\EX(D)$, outputs a
hypothesis $h$, that satisfies the following with probability at least $1 -
\delta$, 
\begin{enumerate}
\item $\falsen(h, D) \leq \epsilon$
\item $\falsep(h, D) \leq \opt^- + \epsilon$, where $\opt^- =
\displaystyle\min_{c \in C^-(D)} \falsep(c, D)$
\end{enumerate}
We refer to $\eps$ as the \emph{error parameter} of the learning algorithm.
\end{definition}

\citet{KKM:2012} also define a notion of \emph{fully reliable learning}. Here,
the learning algorithm may output a \emph{partial classifier} $h : X \rightarrow
\{-1, ?, +1\}$, and must make (almost) no errors, while simultaneously
minimizing the probability of abstaining from prediction, \ie outputting $?$.
Again, recall that we are in the agnostic setting, and let $D$ be an arbitrary
distribution over $X \times \moo$. For some partial classifier,  $h : X
\rightarrow \{-1, ?, +1\}$, let $\err(h, D) = \Pr_{(x, y) \sim D}[h(x) = -y]$
denote the error, and $?(h, D) = \Pr_{(x, y) \sim D}[h(x) = ?]$ denote the
uncertainty of $h$. From a concept class $C$, each pair of concepts defines a
partial classifier, $c_p = (c_+, c_-)$, defined as: $c_p(x) = c_+(x)$, if
$c_+(x) = c_-(x)$, and $c_p(x) = ?$ otherwise. Let $C^f(D) = \{ c_p = (c_+, c_-)
~|~ \err(c_p, D) = 0 \}$ denote the fully reliable partial classifiers with
respect to distribution $D$.  Formally, fully reliable learning is defined as:

\begin{definition}[Fully Reliable Learning~\citep{KKM:2012}] A concept class $C$
is fully reliable learnable, if there exists a learning algorithm that for any
distribution $D$ over $X \times \moo$, any $\epsilon, \delta > 0$, with access
to the example oracle $\EX(D)$, outputs a partial hypothesis $h : X \rightarrow
\{-1, ?, +1\}$, that satisfies the
following with probability at least $1 - \delta$, 
\begin{enumerate}
\item $\err(h, D) \leq \epsilon$
\item $?(h, D) \leq \opt^? + \epsilon$, where $\opt^? = \underset{c_p \in
C^f(D)}{\min}
?(c_p, D)$
\end{enumerate}
We refer to $\epsilon$ as the \emph{error parameter} of the learning algorithm.
\end{definition}

\citet{KKM:2012} showed the following simple result.
\begin{theorem}[\citep{KKM:2012}] 
\label{thm:fullyreliable} If a concept class $C$ is positive and
negative reliable learnable in time $T(n, \epsilon)$ and with sample complexity
$S(n, \epsilon)$, then $C$ is fully reliable learnable in time $O(T(n, \epsilon/4))$
and sample complexity $O(S(n, \epsilon/4))$. 
\end{theorem}

\subsection{Approximating Polynomials}
Throughout, if $p: \{-1, 1\}^n \rightarrow \mathbb{R}$ is a real polynomial,
$\deg(p)$ will denote the total degree of $p$.
Let $f: \{-1, 1\}^n \rightarrow \{-1, 1\}$ be a Boolean function.
We say that a polynomial $p: \{-1, 1\}^n \rightarrow \{-1, 1\}$ 
is an $\eps$-approximation for $f$ 
if $|p(x) - f(x)| \leq \eps$ for all $x \in \{-1, 1\}^n$. 
We let $\adeg_{\eps}(f)$ denote the least 
degree of an $\eps$-approximation for $f$. 
We define $\adeg(f) = \adeg_{1/3}(f)$
and refer to the \emph{approximate degree} of $f$ 
without qualification.
The constant $1/3$
is arbitrary and is chosen by convention.

\subsection{One-sided Approximating Polynomials}
\label{sec:onesidedpoly}

We define the notion of one-sided approximating polynomials. The definitions as
they are presented here essentially appeared in prior work of \cite{BT:2013} (see
also \cite{sherstovthresholddegree}),
who only required the notion we refer to as positive one-sided approximate degree.
Here, we explicitly distinguish between positive and negative one-sided approximations.
  
\begin{definition}[Positive One-Sided Approximating Polynomial] Let $f : \moo^n
\rightarrow \moo$ be a Boolean function. We say that a polynomial $p$ is a
positive one-sided $\epsilon$-approximation for $f$ if $p$ satisfies the
following two conditions.  
\begin{enumerate}
\item For all $x \in f^{-1}(1)$, $p(x) \in [1 - \epsilon, \infty)$
\item For all $x \in f^{-1}(-1)$, $p(x) \in [-1 - \epsilon, -1 + \epsilon]$. 
\end{enumerate}
\end{definition}
  
Analogously, we say that $p$ is a negative one-sided $\epsilon$-approximation
for $f$ if $p$ satisfies:
\begin{enumerate}
\item For all $x \in f^{-1}(1)$, $p(x) \in [1 - \epsilon, 1 + \epsilon]$. 
\item For all $x \in f^{-1}(-1)$, $p(x) \in (-\infty, -1 + \epsilon]$. 
\end{enumerate}

We define the \emph{positive and negative one-sided approximate degrees} of $f$,
denoted $\adegpe(f)$ and $\adegne(f)$ respectively, to be the minimum degree of
a positive (respectively, negative) one-sided $\eps$-approximating polynomial
$p$ for $f$. We define $\adegp:= \widetilde{\deg}_{+, 1/3}$ and $\adegn:= \widetilde{\deg}_{-, 1/3}$, and refer to these quantities as the \emph{positive and negative one-sided 
approximate degrees} of $f$ without qualification.

For a polynomial $p: \{-1, 1\}^n \rightarrow \{-1, 1\}$, we define its \emph{weight} to be the sum of
the absolute values of its coefficients and denote it by $\weight(p)$. Let $C$
be a concept class of Boolean functions; we say that $C$ is positive one-sided
$\epsilon$-approximated by degree $d$ and weight $W$ polynomials, if the
following is true: for every $c \in C$, there exists a polynomial $p$ of
weight at most $W$ and degree at most $d$, such that $p$ is a positive
one-sided $\epsilon$-approximation of $c$.  An analogous definition can be made
for the negative one-sided $\epsilon$-approximation of a concept class.

\subsection{Additional Notation}
Throughout this paper, we use $\tilde{O}$ to hide factors
polylogarithmic in $n$ and $\log(1/\eps)$. We also define
$\sgn(t) = -1$ if $t \leq 0$ and 1 otherwise.

\subsection{Generalization Bounds}
\label{sec:rademacher}

We review the basic results required to bound the generalization error of
our algorithms for reliable agnostic learning.
Let $\classF : X \rightarrow \reals$ be a function class.
Let $\epsilon_1, \dots, \epsilon_n$ independently take values in $\{-1, +1\}$ with equal
probability, and let the variables $x_1, \ldots, x_n$ be chosen i.i.d. from some
distribution $\mu$ over $X$. 
Then the
Rademacher complexity of $\classF$, denoted $\rad_m(\classF)$, is defined as:
\begin{align*}
\rad_m(\classF) &= \E\left[ \sup_{f \in \classF} \frac{1}{n} \sum_{i=1}^m f(x_i)
\epsilon_i \right],
\end{align*}

Rademacher complexities have been widely used in the statistical learning theory
literature to obtain bound on generalization error. Here, we only cite results that are
directly relevant to our work. Suppose $D$ is some distribution over $X \times \{-1,
1\}$. Let $\ell : \reals \times \{-1, 1\} \rightarrow \reals^+$ be a loss
function. For a function, $f : X \rightarrow \reals$, the expected loss is given
by $\Eloss(f) = \E_{(x, y) \sim D}[\ell(f(x), y)]$. For a sample, $\langle (x_i, y_i) \rangle_{i=1}^m$,
let $\hat{\Eloss}(f) = \frac{1}{m}\sum_{i=1}^m \ell(f(x_i), y_i)$ denote the
empirical loss. \citet{BM:2002} proved the following
result:

\begin{theorem}[\citep{BM:2002}]
\label{thm:BM} Let $\ell$ be a Lipschitz loss function (with
respect to its first argument) with Lipschitz parameter $L$, and suppose that $\ell$ is
bounded above by $B$. Then for any $\delta > 0$, with probability at least $1 -
\delta$ (over the random sample draw), simultaneously for all $f \in \classF$,
the following is true: 
\begin{align*}
|\Eloss(f) - \hat{\Eloss}(f)| &\leq 4 L \rad_m(\classF) + 2B
\sqrt{\frac{\log(1/\delta)}{2m}},
\end{align*}
where $\rad_m(\classF)$ is the Rademacher complexity of the function class
$\classF$, and $m$ is the sample size.
\end{theorem}

Finally, let $X = \moo^n$ and let $\Pol_{d, W}$ be the class of $n$-variate polynomials of 
degree at most $d$ and weight at most $W$. Observe that for $x \in X$,
$\linfnorm{x} \leq 1$. Note that we can view $p(x)$ as a linear function in an
expanded feature space of dimension $n^d$, and the $1$-norm of $p$ in such a
space is bounded by $W$. \citet{KST:2008} proved the following result:
\begin{theorem}[\citep{KST:2008}] Let $X$ be an $n$ dimensional instance space
and ${\mathcal W} = \{ w ~|~ w(x) \mapsto w \cdot x\}$ be a class of linear
functions, such that for each $x \in X$, $\linfnorm{x} \leq 1$, and for each $w
\in {\mathcal W}$, $\lonenorm{w} \leq W$, then, $\rad_m({\mathcal W}) \leq W
\sqrt{\frac{2 \log(2n)}{m}}$. 
\end{theorem}

In our setting, the above implies that the Rademacher complexity of $\Pol_{d,
W}$ is bounded as follows:
\begin{align}
\rad_m(\Pol_{d, W}) \leq W \sqrt{\frac{2d \log(2n)}{m}}. \label{eqn:radcomppoly}
\end{align}

\section{Learning Algorithms}
\label{sec:results}
We first present a very simple algorithm for positive reliable learning
disjunctions in Section~\ref{sec:simple-alg}. It is unlikely, however, that such
simple algorithms for reliable learning exist for richer classes; in
Section~\ref{sec:lp}, we present our main result deriving reliable learning
algorithms from one-sided polynomial approximations.

\subsection{A Simple Algorithm for Positive Reliably Learning Disjunctions}
\label{sec:simple-alg}

The learning algorithm (presented in Fig.~\ref{fig:algdisj}) ignores all
positive examples and finds a disjunction that is maximally positive and
classifies all the negative examples correctly (see also \citet[Chap.
1]{KV:1994}).

\begin{figure}
\begin{center}
\fbox{
\begin{minipage}{\textwidth}
\noindent{\bf Input}: Sample $\langle (x_i, y_i) \rangle_{i=1}^m$ from $D^m$

\begin{enumerate}
\item Let $h = x[1] \vee \bar{x}[1] \vee \cdots x[n] \vee \bar{x}[n]$ be the disjunction
that include all literals
\item For every $(x_i, y_i)$ such that $y_i = -1$, for $j = 1, \ldots, n$,
modify $h$ by dropping the literal $x[j]$ if $x_i[j] = 1$ and the literal
$\bar{x}[j]$ if $x_i[j] = -1$
\item Output $h$
\end{enumerate}
\end{minipage}
}
\end{center}
\caption{\label{fig:algdisj} Algorithm: Positive Reliable Learning Disjunctions}
\end{figure}

\begin{theorem} The algorithm in Fig.~\ref{fig:algdisj} positive reliably learns
the class of disjunctions for some $m$ in $O(n/\epsilon^2)$, where $m$ is the
number of labeled examples that the algorithm takes as input.
\end{theorem}
\begin{proof} Let $\DISJ$ denote the class of disjunctions and let $D$ be the
distribution over $X \times \moo$. It is known that $\VC(\DISJ) = n$, and hence
for some $m = O(n/\epsilon^2)$, the following is true for every $c \in \DISJ$:
\begin{align*}
|\falsep(c; D) - \frac{1}{m} \sum_{i: y_i = -1} \eye(c(x_i) = +1)| &\leq
\epsilon/2, \\
|\falsen(c; D) - \frac{1}{m} \sum_{i: y_i = +1} \eye(c(x_i) = -1)| &\leq
\epsilon/2. 
\end{align*}
Recall that $\DISJ^+(D)$ denotes the positive reliable disjunctions for
distribution $D$. Let $c^*_+ \in \DISJ^+(D)$ be such that $\falsen(c^*_+) =
\min_{c \in \DISJ^+(D)} \falsen(c)$. Both $h$ and $c^*_+$ classify all the
negative examples in the sample correctly; since $h$ is chosen to have the
largest number of literals subject to this property, it is the case that $(1/m)
\sum_{i : y_i = +1} \eye(h(x_i) = -1) \leq \sum_{i : y_i = +1} \eye(c^*_+(x_i) =
-1)$. Then, we have
\begin{align*}
\falsep(h) &\leq \frac{1}{m} \sum_{i : y_i = -1} \eye(h(x_i) = +1) + \epsilon/2,
 = 0 + \epsilon/2 \leq \epsilon \\
\falsen(h) &\leq \frac{1}{m} \sum_{i : y_i = +1} \eye(h(x_i) = -1) + \epsilon/2.
\\
&\leq \frac{1}{m} \sum_{i : y_i = +1} \eye(c^*_+(x_i) = -1) + \epsilon/2 \leq
\falsen(c^*_+) + \epsilon
\end{align*}
\end{proof}

\subsection{From One-Sided Approximations to Reliable Learning} 
\label{sec:lp}

In this section, we prove our main learning result. We describe a generic
algorithm that positive reliably learns any concept class that can be positive
one-sided approximated by degree $d$ and weight $W$ polynomials. The weight $W$
controls the sample complexity of the learning algorithm, and the degree $d$
controls the running time. For many natural classes, the resulting algorithm is
has strong attribute-efficient properties, since the weight of the approximating
polynomial typically depends only on the number of \emph{relevant} attributes. 

Our algorithm extends the $L_1$-regression technique of \citet{KKMS:2005} for
agnostic learning, but we require a more detailed analysis. In the case of
positive-reliable learning, it is required that the hypothesis output by the
algorithm makes almost no false positive errors --- this is enforced as
constraints in a linear program. To control the false negative errors of the
hypothesis, we have the objective function of the linear program minimize the
\emph{hinge loss}, which is analogous to the $L_1$ loss, but the penalty is only
enforced when the prediction disagrees in sign with the true label. To bound the
generalization error of the output hypothesis, we use bounds on the Rademacher
complexity of the approximating polynomials (see Section~\ref{sec:rademacher}
for details).

\begin{theorem} \label{thm:lp} Let $C$ be a concept class that is positive
(negative) one-sided $\epsilon$-approximated by polynomials of degree $d$ and
weight $W$.  Then, $C$ can be positive (negative) reliably learned by an
algorithm with the following properties:
\begin{enumerate}
\item The running time of the learning algorithm is polynomial in $n^d$ and
$1/\epsilon$.
\item The sample complexity is $m = \max\{\frac{512}{\epsilon^4} \cdot W^2 d
\log(2n) , \frac{64}{\epsilon^2} (W + 1)^2 \log\left(\frac{1}{\delta}\right) \}$ 
\item The hypothesis output by the algorithm can be evaluated
at any $x \in X$ in time $O(n^d)$. 
\end{enumerate}
\end{theorem}

\begin{proof} We only prove the theorem for the case of positive reliable
learning. The case of negative reliable learning is entirely symmetric.  \smallskip \\
\noindent \textbf{Description of Algorithm.}
Suppose $D$ is an arbitrary distribution over $X \times \moo$ and let $\sample =
\langle (x_i, y_i) \rangle_{i=1}^m$ be a sample drawn according to $D$. The
learning algorithm first solves the following mathematical program.

\begin{center}
\boxed{
\begin{minipage}{0.6 \textwidth}
\[
\begin{array}{ll}
\underset{p~:~\operatorname{deg}(p) \leq d}{\operatorname{minimize}}~~ &
\displaystyle\sum_{i : y_i = +1} (1 - p(x_i))_+ \\
\mbox{subject to } & \\
& p(x_i) \leq -1 + \epsilon \quad \forall i \mbox{ such that } y_i = -1 \\
& \operatorname{weight}(p) \leq W.
\end{array}
\] 

\end{minipage}
}
\end{center}

Here $(a)_+$ denotes $a$ if $a > 0$ and $0$ otherwise. This program is similar
to one used in the $L_1$-regression algorithm for agnostic learning introduced by
\citet{KKMS:2005}. The variables of the program are the $\sum_{j = 0}^d {n \choose j}=O(n^d)$ coefficients 
of the polynomial $p(x)$. The above mathematical program is
then easily implemented as a linear program.

Let $p$ denote an optimal solution to the linear program. The hypothesis output
by the algorithm will be a randomized boolean function, defined as follows:
\begin{enumerate}
\item If $p(x) \leq -1$, $h(x) =  -1$.
\item If $p(x) \geq 1$,  $h(x) = +1$.
\item If $-1 < p(x) < +1$, $h(x) = \begin{cases} +1 & \mbox{with probability} (1
+ p(x))/2 \\
-1 & \mbox{with probability} (1 - p(x))/2 \end{cases}$
\end{enumerate} \smallskip

\noindent \textbf{Running Time.}
Since the above program can be implemented as a linear program with $O(n^d)$
variables and $O(m + n^d)$ constraints, the running time to produce the output polynomial
$p$ is $\poly(m, n^d)$. Note that the polynomial $p$ defines the output
hypothesis $h$ completely, except for the randomness used by $h$. For any $x$, $h(x)$ can be
evaluated in time $O(n^d)$ by a randomized Turing machine.
Remark~\ref{rem:rand2det} explains how $h$ can be converted to a deterministic
hypothesis. \smallskip
%
%

\noindent \textbf{Generalization Error.}
We will use two loss functions in our analysis.
Define $\ell_+ : \reals \times \moo \rightarrow \reals^+$ as follows:
\begin{align*}
\ell_+(y^\prime, +1) &= 0, \\
\ell_+(y^\prime, -1) &= 
		\begin{cases} 
			0 & y^\prime \leq -1 + \epsilon \\ 
			\frac{1}{\epsilon}(y^\prime + 1 - \epsilon) & -1 + \epsilon < y^\prime
			\leq -1 + 2 \epsilon \\
			1 & -1 + 2 \epsilon < y^\prime
		\end{cases}
\end{align*}
Clearly $\ell_+$ is bounded between $[0, 1]$ always and also it is
$1/\epsilon$-Lipschitz. For a function, $f : X \rightarrow \reals$, let
$\Eloss_+(f)$ denote the expected loss of $f$ under $D$ and the loss function
$\ell_+$, and similarly let
$\hat{\Eloss}_+(f)$ denote the empirical loss of $f$ under $\ell_+$.

Define $\ell_- : \reals \times \moo \rightarrow \reals^+$ as follows:
\begin{align*}
\ell_-(y^\prime, -1) &= 0, \\
\ell_-(y^\prime, +1) &= (1 - y^\prime)_+
\end{align*}
Let $p$ continue to denote an optimal solution to the linear program.  Note that
since $X = \moo^n$, and $\operatorname{weight}(p) \leq W$, it holds that $|p(x)|
\leq W$ for all $x \in X$.  It follows that $\ell_-(p(x), b) \leq W + 1$ for all
$x \in X$ and $b \in \{-1, +1\}$.  Moreover, $\ell_-$ is easily seen to be
$1$-Lipschitz. For a function, $f : X \rightarrow \reals$, let $\Eloss_-(f)$ and
$\hat{\Eloss}_-(f)$ denote the expected and empirical loss of $f$ respectively
under distribution $D$ and loss function $\ell_-$. \smallskip

Recall that $C^+(D) = \{c \in C ~|~ \falsep(c) = 0\}$. Let $c^* \in C^+(D)$ be an
optimal positive reliable classifier, i.e., $\falsen(c^*) = \displaystyle\min_{c
\in C^+(D)} \falsen(c)$. Let $p^* \in \Pol_{d, W}$ be a positive one-sided
$\epsilon$-approximating polynomial for $c^*$ whose existence
is guaranteed by hypothesis. Note that since $p^*(x) \geq 1 - \epsilon$ for $x \in
(c^*)^{-1}(1)$ and $p^*(x) \in [-1 - \epsilon, -1 + \epsilon]$ for $x \in
(c^*)^{-1}(-1)$, the following is true:
\begin{align*}
\Eloss_+(p^*) &= 0 \\
\Eloss_-(p^*) &\leq 2 \falsen(c^*) + \epsilon
\end{align*}

Here, the inequality holds because $\ell_-(y^\prime, 1) = (1 - y^\prime)_+$,
which is between $2 - \epsilon$ and $2 + \epsilon$ when $p^*(x) \in [-1
-\epsilon, -1 + \epsilon]$. Thus, each $x$ on which $c^*$ makes a false negative
error contributes approximately $2$ to $\Eloss_-(p)$; the extra $\epsilon$
accounts for the approximation error.

Fix a $\delta > 0$. Recall that $\Pol_{d, W}$ is the class of degree $d$ and
weight $W$ polynomials. Then the Rademacher complexity, $\rad_m(\Pol_{d, W})
\leq W \sqrt{(2d \log(2n))/m}$ (see (\ref{eqn:radcomppoly}) in
Section~\ref{sec:rademacher}). Let $\alpha = (4/\epsilon) \rad_m(\Pol_{d, W}) +
2 (W + 1) \sqrt{\frac{\log(1/\delta)}{2m}}$. Recall that $p$ is the polynomial
output by running the linear program. Then the following holds with probability
$1-\delta$:
\begin{align}
\Eloss_-(p) &\leq \hat{\Eloss}_-(p) + \alpha & \mbox{Using Theorem~\ref{thm:BM}}
\nonumber \\
&\leq \hat{\Eloss}_-(p^*) + \alpha & \mbox{Since $p^*$ is a feasible solution}
\nonumber \\
&\leq \Eloss_-(p^*) +  2 \alpha & \mbox{Using Theorem~\ref{thm:BM}} \nonumber \\
&\leq 2 \falsen(c^*) + 2 \alpha + \epsilon. \label{eqn:falsen-bound}
\end{align}

Similarly, using Theorem~\ref{thm:BM} and the fact that $\hat{\Eloss}_+(p) = 0$,
we have that $\Eloss_+(p) \leq \alpha$.

We have the following: 
\begin{align*}
\falsep(h) &= \E_{(x, y)\sim D} [\eye(y = -1) \eye(h(x) = 1)] = \E_{(x, y) \sim D}[\eye(y=-1) \Pr(h(x) = 1 ~|~ p(x))].\\
\intertext{The inner probability is only over the randomness used by the
hypothesis $h$. It follows from the definition of the randomized hypothesis $h$
and the loss function $\ell_+$, that $\Pr(h(x) = 1 ~|~ p(x)) \leq \ell_+(p(x),
-1) + \epsilon/2$. This together with the fact that $\ell_+(p(x), +1) = 0$ for
all $x$, and $\Eloss_+(p) \leq \alpha$, gives us} 
\falsep(h) &\leq \E_{(x, y) \sim D} [\epsilon/2 + \ell_+(p(x), y)] \leq \epsilon/2 + \Eloss_+(p) \leq \epsilon/2 + \alpha.
\end{align*}

Similarly, we have the following:
\begin{align*}
\falsen(h) &= \E_{(x, y) \sim D}[\eye(y = +1) \eye(h(x) = -1)]  = \E_{(x, y)
\sim D}[\eye(y = +1) \Pr(h(x) = -1 ~|~ p(x))] \\
\intertext{Again, the inner probability is only over the randomness of the
hypothesis $h$. From the definitions of $\ell_-$ and $h$, it follows that
$\Pr(h(x) = -1 ~|~ p(x)) \leq \ell_-(p(x), +1)/2$. Using this along with the
fact that $\ell_-(p(x), -1) = 0$ for all $x$, and (\ref{eqn:falsen-bound}) we get}
\falsen(h) &\leq \E_{(x, y) \sim D}[\frac{1}{2} \ell_-(p(x), y)] \leq \falsen(c^*) +  \alpha + \epsilon/2 
\end{align*}

Finally, it is easily verified that for the value of $m$ in the theorem
statement, $\alpha \leq \epsilon/2$. This completes the proof of the theorem.
\end{proof}
\begin{remark} \label{rem:rand2det} The randomized hypothesis $h$ can easily be converted to a
deterministic one as follows: let $H(x) = \chop(p(x))$, where $\chop(a) = a$ for
$a \in [-1, 1]$ and $\chop(a) = \sign(a)$ for $a \not \in [-1, 1]$. Note
that $\E[h(x)] = H(x)$ for each $x$. Take a fresh sample of size $m =
O(1/\epsilon^2)$ and construct $\langle (H(x_i), y_i) \rangle_{i=1}^m$. For a
threshold $t$, let $h_t = \sign(H(x) - t)$. Find the smallest value $t^*$, such
that $\frac{1}{m} \sum_{y_i = -1} \eye (h_{t^*}(x_i) = +1) \leq \epsilon$. Then,
a simple VC argument implies that $h_{t^*}$ is a deterministic hypothesis
with the required properties.
\end{remark}

Theorem~\ref{thm:lp} satisfies a strong \emph{attribute-efficiency} property.
The sample complexity depends only logarithmically on $n$, and polynomially on
the weight of the polynomial approximations, which can be much smaller then
$n^d$. A similar statement can also be made for agnostic learning; this
observation was already implicit in some prior work (see \eg \citep{FKV:2013});
we state this as a theorem for completeness.  Instead of the mathematical
program described in the proof of Theorem~\ref{thm:lp}, to obtain Theorem
\ref{thm:agnostic}, the $L_1$-regression algorithm of \citet{KKMS:2005} is
directly applied, with the added constraint that the weight of the approximating
polynomial is at most $W$. The rest of the proof is similar, but simpler --- we
only use $\ell(y^\prime, y) = |y^\prime - y|$ as the loss function in the
analysis.  The proof is omitted since it is essentially a simplification of the
proof of Theorem~\ref{thm:lp}.

\begin{theorem}\label{thm:agnostic} Let $C$ be a concept class of functions from
$X \rightarrow \moo$, such that for every $c \in C$, there exists a polynomial
$p$ of degree at most $d$ and weight at most $W$, such that for all $x \in X$,
$|p(x) - c(x)| \leq \epsilon$. Then, $C$ can be agnostically learned with the
following properties:
\begin{enumerate}
\item The running time of the learning algorithm is polynomial in $n^d$ and
$1/\epsilon$.
\item The sample complexity is polynomial in $W$, $\log(n)$, $\log(1/\delta)$
and $1/\epsilon$. 
\item The hypothesis output by the algorithm can be evaluated at any $x \in X$
in time $O(n^d)$. 
\end{enumerate}
\end{theorem}

\eat{In the next section, we show how specific classes can \vknote{Justin you fill
this up.. either at the end of this section --- or at the beginning of the
next.}}

\section{One-sided Polynomial Approximations}
\label{sec:one-sided}
In this section, we construct both positive and negative one-sided polynomial
approximations for low-weight halfspaces, as well as positive (respectively,
negative) one-sided approximations for disjunctions (respectively, conjunctions)
of low-weight halfspaces.

\begin{theorem}
\label{thm:lowweighthalfspace}
Let $h(x) = \sgn( w_0 + \sum_{i=1}^n w_i x_i)$ denote any halfspace, where $w_i$
are integers. Let $W = \sum_{i=0}^n |w_i|$ denote the \emph{weight} of $h$.
Both $\adegpe(h)$ and $\adegne(h)$ are in
$\tilde{O}\left(\sqrt{W\log\left(1/\eps\right)}\right)$, with the relevant
approximating polynomials having weight at most
$\exp\left(\tilde{O}\left(\sqrt{W\log\left(1/\eps\right)}\right)\right)$.  In
particular, the majority function $\MAJ(x) = \sgn( \sum_{i=1}^n x_i)$ has both
positive and negative $\eps$-approximating polynomials of degree at most
$\tilde{O}(\sqrt{n\log\left(1/\eps\right)})$ and weight at most
$\exp\left(\tilde{O}(\sqrt{n\log\left(1/\eps\right)})\right)$ 
\end{theorem}

\medskip
\noindent \textbf{Remark.}
By adapting standard symmetrization
arguments (cf. \cite{smallerrorquantum}), the $\tilde{O}(\sqrt{n\log\left(1/\eps\right)})$ upper bound on $\adegpe(\MAJ)$ 
is easily seen to be tight up to factors hidden by the $\tilde{O}$ notation.\\

\begin{proof} 
We begin with the case of constant $\eps$; i.e., we first show that 
for $\eps=1/4$, 
 $\adegpe(h)$ and $\adegne(h)$ are in $O(W^{1/2})$. 
We use the following standard properties of the Chebyshev polynomials (cf. the standard texts of 
\cite{cheney} and \cite{rivlin}). 

\begin{fact}
The $d$'th Chebyshev polynomial of the first kind, $T_d(t): \mathbb{R} \rightarrow \mathbb{R}$
has degree $d$ and satisfies
\begin{align}
\label{prop2} & |T_d(t)| \leq 1  \mbox{ for all } -1 \leq t \leq 1.\\
\label{prop3} & 2 \leq T_{\lceil a \rceil}(1+ 1/a^2)  \mbox{ for all } a \geq 1. \\
\label{prop4} & T_d(t) \mbox{ is non-decreasing on the interval } [1, \infty].  \\
\label{prop5} & \mbox{All coefficients of } T_d \mbox{ are bounded in absolute value by } 3^d. 
\end{align}
\end{fact}

Let $d=\lceil W^{1/2} \rceil$. Consider the univariate polynomial 
$G(t) = T_{d}(2t/W + 1)$. Then $G$ satisfies the following properties.
\begin{align}
\label{prop1g} & G(t) \in [-1, 1] \mbox{ for all } t \in [-W, 0].   \\
\label{prop2g} & G(t) \geq 2 \mbox{ for all } t \in [1, \infty]. 
\end{align}

Indeed, Property \ref{prop1g} follows from Property \ref{prop2},
while Property \ref{prop2g} follows from Properties \ref{prop3} and \ref{prop4}.

Now consider the univariate polynomial $P(t) = G(t)^4/4 - 1$. It is straightforward to check that
\begin{align}
\label{prop1P} & P(t) \in [-3/4, 1] \mbox{ for all } t \in [-W, 0].   \\
\label{prop2P} & P(t) \geq 3 \mbox{ for all } t \in [1, \infty].
\end{align}

Finally, consider the $n$-variate polynomial $p : \{-1, 1\}^n \rightarrow \mathbb{R}$
defined via $$p(x) = P(w_0 + \sum_{i=1}^n w_i x_i).$$ Combining the
fact that $\sum_{i=0}^n |w_i| \leq W$ with
Properties \ref{prop1P} and \ref{prop2P}, we see that
$p$ is a positive one-sided $1/4$-approximation for $h$. Moreover,
$\deg(p) \leq \deg(P) = O(W^{1/2})$, and the weight of $p$
is at most $W^{O(\sqrt{W})}$. Similarly,
$-p(-x)$ is a negative one-sided $1/4$-approximation for $h$. 
This completes
the proof for $\eps=1/4$.

The construction for $\eps=o(1)$ is somewhat more complicated. For any $k \geq 1$ and any $W$,
 \cite{kahn} construct a univariate polynomial $S_k$ 
satisfying the following properties: 
\begin{align}
\label{prop1s}  & \deg(S_k) \leq k. \\
\label{prop2s} & S_k(t) \geq 1 \mbox{ for all } t \geq W.  \\
\label{prop3s} & S_k(t) \leq \exp\left(-\Omega(k^2/W\log W)\right) \mbox{ for all } t \in \{0, \dots, W-1\}.\\
\label{prop4s} & \mbox{All coefficients of } S_k(t) \mbox{ are bounded in absolute value by } 
W^{O(k)}.   
\end{align}

For completeness, we give the details of this construction and a proof of Properties \ref{prop1s}-\ref{prop4s} in Appendix \ref{app:kahn}. 

For any $\eps > 0$, let $k = \lceil \left(W\log W \log\left(1/\eps\right)\right)^{1/2}\rceil$, and
let $q: \{-1, 1\}^n \rightarrow \mathbb{R}$ denote the $n$-variate 
polynomial defined via $$q(x) = S_k\left(W+w_0 + \sum_{i=1}^n w_i x_i\right).$$
It is then straightforward to check that 
$q$ is a positive one-sided $\eps$-approximation for $h$ of degree
at most $k=\tilde{O}\left(\sqrt{W\log\left(1/\eps\right)}\right)$ and weight at most
$W^{\tilde{O}(k)}$. Similarly,
$-q(-x)$ is a negative one-sided $\eps$-approximation for $h$. 
This completes the proof.
\end{proof}

\eat{
\medskip
\noindent \textbf{Remark.}
We suspect that the $\tilde{O}$'s appearing in the statement of Theorem \ref{thm:lowweighthalfspace} can be replaced by $O$'s, by replacing the polynomial $S_k$
described by Kahn et al. \cite{kahn} with a polynomial constructed by de Wolf \cite{dewolf} using
quantum algorithms. 
However, we have not verified the details. }

The concept class of majorities is defined as the collection of the majority
functions on each of the $2^n$ subsets of the variables. 

\begin{corollary} \label{cor:halfspace} The concept class of Majorities on $n$
variables can be positive or negative reliably agnostically learned with error
parameter $\eps$ in time
$2^{\tilde{O}\left(\sqrt{n\log\left(1/\eps\right)}\right)}$. 
\end{corollary}
\begin{proof}
Combine Theorems \ref{thm:lp} and \ref{thm:lowweighthalfspace}, noting that any majority function is a halfspace of weight at most $n$.
\end{proof}

By combining Corollary \ref{cor:halfspace} with Theorem \ref{thm:fullyreliable},
we obtain a \emph{fully reliable} algorithm for learning low-weight halfspaces.

\begin{corollary} 
The concept class of Majorities on $n$ variables can be fully reliably
learned with error parameter $\eps$ in time $2^{\tilde{O}\left(\sqrt{n\log\left(1/\eps\right)}\right)}$. 
\end{corollary}

\noindent We now consider significantly more expressive concept classes: \emph{disjunctions and conjunctions} of majorities.

\begin{theorem}
\label{thm:compose}
Consider $m$ functions $f_1 \dots f_m$. 
Fix a $d > 0$, and suppose that each $f_i$ has a positive one-sided $(\eps/m)$-approximating polynomial
of degree at most $d$ and weight at most $W$.  
Then $\OR_m(f_1, \dots, f_m)$ has a positive one-sided $\eps$-approximating polynomial of 
degree at most $d$ and weight at most $m \cdot W$. 

Similarly, if each $f_i$ has a negative one-sided $(\eps/m)$-approximating polynomial
of degree at most $d$ and weight at most $W$,
then $\AND_m(f_1, \dots, f_m)$ has a negative one-sided $\eps$-approximating polynomial of 
degree at most $d$ and weight at most $m \cdot W$. 
\end{theorem}
\begin{proof}
We prove the statement about $\OR_m(f_1, \dots, f_m)$; the statement about $\AND_m(f_1, \dots, f_m)$ is analogous.
Let $p_i$ be a positive one-sided $(\eps/m)$-approximating polynomial for $f_i$.
Then $p=-1+\sum_{i=1}^m (1+p_i)$ is a positive one-sided $\eps$-approximating polynomial for $f$. Moreover, the degree of $p$ is at most $\max_i \{\deg(p_i)\} \leq d$, while the
weight of $p$ is at most $m \cdot W$. This completes the proof.
\end{proof}

\begin{corollary} \label{cor:halfspaceintersect}
Disjunctions of $m$ Majorities can be positive reliably learned with error parameter $\eps$
in time $2^{\tilde{O}(\sqrt{n \log(m/\eps)})}$. Conjunctions of $m$ Majorities can also be negative reliably
learned in the same time bound.
\end{corollary}
\begin{proof}
Combine Theorems \ref{thm:lp}, \ref{thm:lowweighthalfspace}, and \ref{thm:compose}.
\end{proof}

\section{Trading off Runtime for Sample Complexity}
\label{sec:tradeoff}
\subsection{Standard Agnostic Learning of Conjunctions}
\cite{KKMS:2005} showed how to use $L_1$-regression to 
agnostically learn conjunctions on $n$ variables in time $2^{\tilde{O}(\sqrt{n \log(1/\eps)})}$. 
However, the sample complexity of the algorithm can also
be as large as $2^{\tilde{O}(\sqrt{n \log(1/\eps)})}$.
This result relies on the existence of $\eps$-approximating
polynomials for the $n$-variate $\AND$ function of degree $\tilde{O}(\sqrt{n \log(1/\eps)})$.

Theorem~\ref{thm:agnostic} gives an avenue for obtaining better sample
complexity, at the cost of increased runtime: if we can show that any
conjunction on $n$ variables can be $\eps$-approximated by a degree $d$
polynomial of weight $W \ll 2^{\sqrt{n \log(1/\eps)}}$, then the
$L_1$-regression algorithm will have sample complexity only $\poly(d, W)$ and
runtime $n^{O(d)}$.  Thus, in order to obtain tradeoffs between runtime and
sample complexity for algorithms that agnostically learn conjunctions, it
suffices to understand what are the achievable tradeoffs between degree and
weight of $\eps$-approximating polynomials for the $\AND$ function.

In fact, this question is already well-understood in the case of constant $\eps$: 
letting $\AND_n$ denote the
$\AND$ function on $n$ variables, 
\cite{servediotanthaler} implicitly showed that 
for any $\sqrt{n} < d$ and any $\eps = \Theta(1)$, 
there exists an $\eps$-approximating
polynomial for the $\AND_n$ function of degree $d$
and weight $\poly(n) \cdot 2^{\tilde{O}(n/d)}$. In fact, this construction is essentially optimal,
matching a lower bound for constant $\eps$ proved in the same paper (see also \cite[Lemma 20]{BT:2013}). 
We now extend the ideas of \cite{servediotanthaler} to handle subconstant values of $\eps$.

\eat{
\begin{corollary}
For any $d > \sqrt{n}$ and $\eps = \Theta(1)$, the class of conjunctions on $n$ variables can be agnostically learned to error $\eps$ 
in time $n^{O(d)}$, with sample complexity $\poly(n) \cdot 2^{\tilde{O}(n/d)}$.
\end{corollary}
}

\begin{theorem}\label{thm:tradeoff}
Fix a $d > \tilde{\Omega}\left(\sqrt{n \log n} \log(1/\eps)\right)$. There exists an (explicit) 
$\eps$-approximating
polynomial for $\AND_n$ of degree $d$
and weight $2^{\tilde{O}(n \log(1/\eps)/d)}$.
\end{theorem}
\begin{proof}
We write $\AND_n$ as an ``and-of-ands'', where the outer $\AND$
has fan-in $t$, and the inner $\AND$s each have fan-in $n/t$,
where we choose $t$ such that $t/\log t = n^2\log(1/\eps)/d^2$. 
That is, we write $\AND_n(x) = \AND_t( \AND_{n/t}(x^{(1)}), \dots, \AND_{n/t}(x^{(t)}))$,
where $x^{(i)} = (x_{n\cdot (i-1)/t+1}, \dots, x_{n\cdot i/t})$ denotes
the $i$th ``block'' of variables in $x$. Note that $t \leq n$ by the assumption
that $d > \tilde{\Omega}\left(\sqrt{n \log n} \log(1/\eps)\right)$. 

We obtain an $\eps$-approximating polynomial $p$ for $\AND_n$
as follows. \cite{kahn}
gave an explicit $\eps$-approximating polynomial $p_t$
for $\AND_t$ of degree $d'=O(\sqrt{t \log t \log(1/\eps)})$. It is an immediate
consequent of Parseval's inequality that $p_t$
has weight at most $t^{d'/2}$. 
We will also need the following standard fact.

\begin{fact}
The real polynomial $q : \{-1,1\}^{n/t} \rightarrow \{-1, 1\}$
defined via $q(y_1, \dots y_{n/t}) = 2 \prod_{i=1}^{n/t} \frac{1+y_i}{2} - 1$
computes $\AND_{n/t}(x)$. Moreover, $q$ has degree at most $n/t$ and weight at most $3$.
\end{fact}

Finally, we define $p(x) = p_t(q(x^{(1)}), \dots q(x^{(t)}))$. 
Notice that $p$ has degree at most
$d' \cdot n/t$ $=$\\ $O\left(\sqrt{t \log t \log(1/\eps)} \cdot n/t\right)$ $=$ $O\left(n \sqrt{\log t \log(1/\eps)/ t}\right)= O(d)$
and weight at most $t^{O(d')} = 2^{\tilde{O}\left(n \log(1/\eps)/d\right)}$ as claimed.
\end{proof}

We obtain the following learning result that holds even for $\eps=o(1)$.

\begin{corollary}
For any $d > \tilde{\Omega}\left(\sqrt{n \log n} \log(1/\eps)\right)$ and $\eps$, the class of conjunctions on $n$ variables can be agnostically learned to error $\eps$ 
in time $n^{O(d)}$, with sample complexity $2^{\tilde{O}(n \log(1/\eps)/d)}$.

\end{corollary}

\subsection{Positive Reliable Learning of DNFs}
As discussed in Section \ref{sec:related}, the reductions of \cite{KKM:2012},
combined with the agnostic learning algorithm for conjunctions due to \cite{KKMS:2005},
imply that DNFs can be
positive reliably learned in time $2^{(\tilde{O}(\sqrt{n}))}$. However, 
the sample complexity of the resulting algorithm may be as large as its runtime. 
Here, we give an algorithm for positive reliable learning of DNFs that has smaller sample
complexity, at the cost of larger runtime. 

\begin{theorem}
\label{thm:tradeoffdnf}
For any DNF $F$ of size $m$ and width (i.e., maximum term length) at most $w$, 
and any $d > \tilde{\Omega}\left(\sqrt{w \log w} \log(1/\eps)\right)$, there exists an (explicit) 
positive one-sided $\eps$-approximating
polynomial for $F$ of degree $d$
and weight $2^{\tilde{O}(w \log(m/\eps)/d)}$. Similarly,
any CNF $F$ of size $m$ and width at most $w$ has a negative one-sided 
$\eps$-approximation with the same weight and degree bounds.
\end{theorem}
\begin{proof}
We prove the result for DNFs; the case of CNFs is analogous. 
Let $C_i$ denote the $i$th clause of $F$. Since $C_i$ has width at most $w$, 
Theorem \ref{thm:tradeoff} implies the existence of an $\eps/m$-approximating polynomial 
$p_i$ for $C_i$ of degree $d$ and weight at most $2^{\tilde{O}(w \log(m/\eps)/d)}$.
Then $p=-1+\sum_{i=1}^m (1+p_i)$ is a positive one-sided $\eps$-approximating polynomial for $F$. Moreover, the degree of $p$ is at most $\max_i \{\deg(p_i)\}\leq d$, while the
weight of $p$ is at most $m \cdot 2^{\tilde{O}(w \log(m/\eps)/d)} = 2^{\tilde{O}(w \log(m/\eps)/d)}$. This completes the proof.
\end{proof}

We obtain the following learning result as a corollary. 

\begin{corollary}
\label{cor:dnf}
For any $d > \tilde{\Omega}\left(\sqrt{w \log w} \log(1/\eps)\right)$,
the concept class of DNFs of size $m$ and width at most $w$ can be positive reliably learned in time $n^{O(d)}$, using at most $2^{\tilde{O}(w \log(m/\eps)/d)}$ samples.
The class of CNFs of size $m$ and width at most $w$ can be negative reliably
learned with the same efficiency guarantees. 
\end{corollary}
\begin{proof}
Combine Theorems \ref{thm:lp} and \ref{thm:tradeoffdnf}.
\end{proof}

\section{Limitations of Our Techniques}
\label{sec:limitations}
\subsection{On Halfspaces} Theorem \ref{thm:lowweighthalfspace} establishes that 
all \emph{low-weight} halfspaces (i.e., weight $o(n^{2-\delta})$ for some $\delta > 0$) can be (both positive and negative) reliably learned 
in time $2^{o(n)}$. 
It is reasonable to ask whether we can reliably learn \emph{all}
halfspaces in time $2^{o(n)}$ using our techniques. Unfortunately, the answer is no.

\begin{theorem} \label{thm:halfspacelb}
There exists a halfspace $h$ for which $\adeg_{+, 1/8}(h)$ and $\adeg_{-, 1/8}(h)$ are both $\Omega(n)$.
\end{theorem}
\begin{proof}
We prove the statement about $\adeg_{+, 1/4}$, as the case of $\adeg_{-, 1/4}$
is similar. 

Given a Boolean function $h : \{-1, 1\}^n \rightarrow \{-1, 1\}$, let $g_h: \{-1, 1\}^{2n} \rightarrow \{-1, 1\}$ 
denote the function $g(x, y) = h(x_1) \cap h(x_2)$, where $x_1, x_2 \in \{-1, 1\}^n$. 
That is, $g$ computes the intersection of two copies of $h$, where the two copies are
applied to disjoint sets of input variables. \cite{sherstovhalfspaces} proved that there exists a halfspace $h$ such that $\deg_{\pm}(g) = \Omega(n)$. Here,
$\deg_{\pm}(g)$ denotes 
the least degree of a real polynomial $p$ that agrees in sign with $g$ at all Boolean inputs. 
Notice $\deg_{\pm}(g) \leq \deg_{+, \eps}(g)$ for any function $g$ and any $\eps < 1$. 

Combining Sherstov's lower bound with Theorem \ref{thm:compose} implies that $\Omega(n) = \deg_{\pm}(g) \leq \adeg_{+, 1/4}(g) \leq \deg_{+, 1/8}(h)$.  
This completes the proof.
\end{proof}

\subsection{On DNFs}
All polynomial-sized DNFs can be positive reliably learned in time and sample complexity $2^{\tilde{O}(\sqrt{n})}$,
and Corollary \ref{cor:dnf} shows how to obtain smooth tradeoffs between runtime and sample complexity for this learning task.
It is natural to ask whether DNFs can be \emph{negative} reliably learned with similar efficiency using our techniques. Unfortunately, this is not the case.
\cite{BT:2013}, extending a seminal lower bound of \cite{aaronsonshi}, showed that there is a polynomial-sized DNF $f$
(more specifically, $f$ is the negation of the {\sc Element Distinctness} function) 
satisfying $\adegn(f) = \Omega((n/\log n)^{2/3})$;
thus, our techniques cannot negative reliably learn polynomial-sized DNFs in time better
than $\exp\left(\tilde{O}\left(n^{2/3}\right)\right)$. 

While Bun and Thaler's is the best-known lower bound on the negative one-sided approximate degree of any polynomial-sized DNF -- indeed, up to polylogarithmic factors, it is the best-known lower bound for any function in $\text{AC}^0$ -- no $o(n)$ upper bound is known for the negative one-sided approximate degree of polynomial-sized DNFs.

\section{Discussion}
We have shown that concept classes with low one-sided approximate degree
can be efficiently learned in the reliable agnostic model. 
As we have seen, one-sided approximate degree is an intermediate notion 
that lies between threshold degree and approximate degree; we have identified important concept classes, such as majorities and intersections of majorities, whose one-sided approximate degree
is strictly smaller than its approximate degree. Consequently, 
we have obtained reliable (in some cases, even \emph{fully reliable}) agnostic 
learning algorithms that are strictly more efficient than the fastest known agnostic ones.
We have thereby given the first evidence that even fully reliable agnostic learning
may be strictly easier than agnostic learning.

The notion of one-sided polynomial approximation has only been introduced very recently
(\cite{BT:2013}), and previously had only been used to prove lower bounds.  
By giving the first \emph{algorithmic} application of one-sided polynomial approximations, our work lends further credence to the notion that these approximations are fundamental objects worthy of further study in their own right. Just as threshold degree and approximate degree have found applications (both positive and negative) in many domains outside of learning theory,
we hope that one-sided approximate degree will as well.
Identifying such applications is a significant direction for further work.

Our work does raise several open questions specific to one-sided polynomial approximations. 
Here we highlight two. 
We have shown that halfspaces of weight at most $W$ have one-sided approximate degree $\tilde{O}(W^{1/2})$, and yet there exist halfspaces with one-sided approximate degree $\Omega(n)$. However,
the (non-explicit) halfspace from \cite{sherstovhalfspaces} that we used to demonstrate the $\Omega(n)$ lower bound has weight $2^{\Omega(n)}$. 
Is it possible that all halfspaces of weight $2^{O(n^{1-\delta})}$ for some $\delta>0$ always have one-sided approximate
degree $o(n)$? 
We also showed how to obtain tradeoffs between the weight and degree of one-sided polynomial approximations for DNFs. Is it possible to obtain similar tradeoffs 
for majorities?

\subsection*{Acknowledgments}

VK is supported by a Simons Postdoctoral Fellowship. JT  is supported by a Simons Research Fellowship. This research was
carried out while the authors were at the Simons Institute for the Theory of
Computing at the University of California, Berkeley.

\appendix
\section{Missing Details For Theorem \ref{thm:lowweighthalfspace}}
\label{app:kahn}

For any $k > 0$, \cite{kahn} define the polynomial $S_k(t)$ as follows
(in the below, $a, b,$ and $r$ are parameters that Kahn et al. ultimately 
set to $a=\Theta(k/\log W)$, $b = \Theta(k^2/(W \log W))$, and $r=k-a-b$). 

\begin{equation}
\label{def:s} S_k(t) = C^{-1} \cdot \left(\prod_{i=0}^a (t-i) \cdot \prod_{j=W-b}^W(t-j)\right) \cdot T_r(\frac{t-a}{W-b-a}),
\end{equation}

where $C= \left( \prod_{i=0}^a (W-i) \cdot \prod_{j=W-b}^W(W-j) \right)\cdot T_r(\frac{W-a}{W-b-a})$
is a normalization constant chosen so that $S_k(W)=1$, and as usual
$T_r$ denotes the $r$'th Chebyshev polynomial of the first kind.

We now verify that $S_k$ satisfies Properties \ref{prop1s}-\ref{prop4s}, which we restate here for the reader's convenience.

\begin{align*}
\mbox{Property } \ref{prop1s}\!:  & \mbox{ } \deg(S_k) \leq k. \\
\mbox{Property } \ref{prop2s}\!: &  \mbox{ }  S_k(t) \geq 1 \mbox{ for all } t \geq W.  \\
\mbox{Property }  \ref{prop3s}\!: &  \mbox{ }  S_k(t) \leq \exp\left(-\Omega(k^2/W\log W)\right) \mbox{ for all } t \in \{0, \dots, W-1\}.\\
\mbox{Property } \ref{prop4s}\!: &  \mbox{ }  \mbox{All coefficients of } S_k(t) \mbox{ are bounded in absolute value by } 
W^{O(k)}.   
\end{align*}

Property \ref{prop1s} is immediate from the definition of $S_k$ and the 
choice of $r=k-a-b$. 

To see that Property \ref{prop2s} holds, we note that $S_k(W)=1$. The property
will therefore follow if we can prove that
\begin{equation}
\label{eq:prop2seq} S_k(t) \geq S_k(W) \mbox{ for all } t \geq W. \end{equation}

To establish Equation \eqref{eq:prop2seq}, note first that $T_r$ is non-decreasing on the interval $[1, \infty]$ (cf. Property \ref{prop4}). Second, notice that $\frac{t-a}{W-b-a}$ is an increasing function on $[W, \infty]$, and is also larger than $1$
on this interval. Thus, $T_r(\frac{t-a}{W-b-a})$ is a non-decreasing function in $t$ for $t \in [W, \infty]$. Finally, it is an easy observation that
$\prod_{i=0}^a (t-i) \cdot \prod_{j=W-b}^W(t-j)$ is a non-decreasing function in $t$ on the interval $t \in [W, \infty]$. Thus, $\prod_{i=0}^a (t-i) \cdot \prod_{j=W-b}^W(t-j) \cdot T_r(\frac{t-a}{W-b-a})$ is a non-decreasing function of $t$ on the same interval, and 
Equation \eqref{eq:prop2seq} follows.

Property \ref{prop3s} is immediate from the analysis of \cite{kahn}.
To see that Property \ref{prop4s} holds, note that $ \prod_{i=0}^a (t-i) \prod^W_{j=W-b}(t-j)$
is a polynomial in $t$ with coefficients all bounded in absolute value by $W^{a+b} \leq W^k$, 
while $T_r(\frac{t-a}{W-b-a})$ is also a polynomial in $t$, with coefficients bounded in absolute value by $(3+a)^r \leq W^k$ (cf. Property \ref{prop5}).

\end{document}